\documentclass[journal,10pt]{IEEEtran}
\usepackage[T1]{fontenc}
\usepackage[utf8]{inputenc}
\usepackage{amsmath,amssymb,amsfonts,bm}
\usepackage{graphicx,booktabs,array,tabularx}
\usepackage{xcolor,cite,url,microtype,flushend,flafter}
\usepackage{etoolbox}
\apptocmd{\thebibliography}{\interlinepenalty=10000\relax}{}{}
\usepackage{tikz}
\usetikzlibrary{arrows.meta,positioning,calc,fit,backgrounds}
\usepackage[hidelinks]{hyperref}
\definecolor{navy}{HTML}{233F5B}
\definecolor{teal}{HTML}{007D80}
\definecolor{disturbance}{HTML}{F18B21}
\newcommand{\R}{\mathbb R}
\newcommand{\SO}{\mathrm{SO}(3)}
\newcommand{\diag}{\operatorname{diag}}
\newcommand{\norm}[1]{\left\|#1\right\|}
\newcommand{\sk}[1]{\widehat{#1}}
\newcommand{\Wset}{\mathcal W}
\newcommand{\Eset}{\mathcal E_w}
\newcommand{\Uset}{\mathcal U_f}
\newcommand{\Vset}{\mathcal V}

\newtheorem{theorem}{Theorem}
\newenvironment{proof}{\noindent\emph{Proof:}}{\hfill$\square$\par}
\title{CBF-Based Control Allocation for\\Omnidirectional Tiltrotor Aerial Robots}
\author{Seobeom~Lee and Hyondong~Oh%
\thanks{Seobeom Lee and Hyondong Oh are with the Department of Mechanical Engineering,
Korea Advanced Institute of Science and Technology (KAIST),
Daejeon, Republic of Korea (e-mail: seobeomlee@kaist.ac.kr; h.oh@kaist.ac.kr).}%
\thanks{Corresponding author: Hyondong Oh (e-mail: h.oh@kaist.ac.kr).}}
\hypersetup{pdfauthor={Seobeom Lee; Hyondong Oh},pdftitle={CBF-Based Control Allocation for Omnidirectional Tiltrotor Aerial Robots}}
\begin{document}
\maketitle
\begin{abstract}
Actively tilting aerial robots can generate six-dimensional wrenches, but slow servo response can make a desired wrench infeasible during reconfiguration. This paper develops a control barrier function (CBF) allocator that preserves wrench feasibility under bounded servo-configuration mismatch. An exact equivalent-wrench identity connects the thrust required at nominal and actual configurations. For a nonsingular six-channel allocation map, inverse allocation reduces robust wrench containment to twelve thrust margins tightened by an ellipsoidal reserve. Their complete derivatives yield a six-variable servo-rate quadratic program, followed by differential thrust allocation for tracking. A conditional forward-invariance result establishes actual-configuration feasibility under reserve coverage, feasible initialization, and admissible CBF inputs. A complementary MuJoCo study evaluates a measured-angle implementation against geometric, augmented differential, and normalized differential allocation in circular and cartwheel maneuvers, with and without disturbances. The proposed implementation achieves lower position error in all four recorded scenarios, with an attitude tradeoff in circular tracking. During the disturbed cartwheel, the desired wrench is feasible at all 1751 stored samples of continuous rotation; baseline violation rates range from 60.8\% to 100\%. These comparisons evaluate sampled feasibility; validation of the nominal-trajectory certificate requires nominal-state and reserve-coverage records.
\end{abstract}
\section{Introduction}
\IEEEPARstart{O}{mnidirectional} aerial robots can regulate position and attitude independently, enabling contact inspection, tool alignment, and six-dimensional physical interaction~\cite{ollero2022,bodie2021,ryll2019}. Fixed tilted-rotor designs obtain this capability through their geometry~\cite{brescianini2016,park2018}, whereas actively tilting platforms can reorient their thrust directions during flight~\cite{ryll2015,allenspach2020}. In the latter case, control allocation must coordinate fast thrust changes with slower servo motion.

Servo transients can make a desired wrench infeasible even when the target pose is kinematically admissible. A force and moment achievable after reorientation may lie outside the thrust-limited set at the current tilt angles. Aggressive motion therefore requires an allocator that preserves wrench authority while both the request and the actuation geometry evolve.

Wrench tracking and instantaneous feasibility are distinct objectives. Tracking reduces the difference between the desired and generated wrench; feasibility requires an admissible thrust vector at the current rotor directions. Servo motion changes the attainable wrench set. We constrain that motion to retain the requested wrench together with a neighborhood accounting for nominal-to-actual configuration mismatch.

Geometric and differential allocators provide direct actuator setpoints, coordinated reconfiguration, and limit handling~\cite{bodie2020,allenspach2020,johansen2013}. Augmented and normalized differential allocation further connect wrench controllers to actuator-rate commands and account for heterogeneous limits~\cite{cuniato2026}. The remaining question addressed here is how to constrain servo reconfiguration so that a specified wrench request remains inside the instantaneous thrust-limited set despite configuration mismatch.

Control barrier functions (CBFs) express this objective as constraints on servo motion~\cite{ames2017,ames2019}. For a nonsingular six-channel configuration, inverse allocation gives six lower and six upper thrust margins. An ellipsoidal wrench reserve tightens them against nominal-to-actual configuration mismatch. Their complete derivatives, including the changing reserve, define a servo-rate quadratic program (QP); a downstream differential allocator commands thrust rates for tracking.

The contributions are threefold:
\begin{enumerate}
\item An exact nominal-to-actual required-thrust identity and a support-based reduction of configuration-robust wrench containment to twelve scalar margins.
\item An explicit servo-rate CBF-QP and a conditional invariance proof connecting nominal margins to desired-wrench feasibility at the actual servo angles, without assuming identical commanded and actual servo velocities.
\item A telemetry-based comparison of a measured-angle CBF variant with geometric allocation (GA), augmented differential allocation (ADA), and normalized differential allocation (NDA) across four scenarios. Recorded force-space views and six-dimensional margins relate tracking performance to instantaneous wrench feasibility.
\end{enumerate}
The guarantee concerns existence of an admissible thrust vector; thrust realization and pose tracking also depend on actuator and closed-loop dynamics.

Sections~\ref{sec:model}--\ref{sec:baselines} establish the model and baseline allocators. Sections~\ref{sec:proposed}--\ref{sec:guarantee} develop the CBF formulation and guarantee. Sections~\ref{sec:simulation}--\ref{sec:conclusion} present the simulation comparison, design implications, and conclusions.

\section{Related Work}
\subsection{Actuation Geometry and Dynamic Allocation}
The attainable wrench set depends on both vehicle design and rotor orientation. Offline tilt optimization improves task-dependent authority~\cite{rajappa2015}, synchronized mechanisms change the flight configuration~\cite{ryll2022}, and independently tilting platforms permit online redistribution of actuator effort~\cite{ryll2015,allenspach2020}. These architectures offer different allocation freedoms. Here, six independently controlled thrust magnitudes form a square wrench map when the six tilt angles are fixed.

Constrained allocation has a substantial history. Active-set methods handle bounded least squares~\cite{harkegard2002}, dynamic QPs shape actuator transients~\cite{harkegard2004}, and null-space methods enforce hard joint limits~\cite{flacco2015}. Dual-layer allocation can also modify the requested virtual control when actuator constraints prevent its realization~\cite{yan2021}. A recent predictive allocator accounts for asymmetric motor dynamics by optimizing null-space motion over a horizon~\cite{pretto2026}. Our method instead constrains the servo motion to preserve a specified neighborhood of the desired wrench inside the attainable set.

For the quantitative comparison, we use geometric allocation (GA), augmented differential allocation (ADA), and normalized differential allocation (NDA). Section~\ref{sec:baselines} gives their essential equations and the differential-allocation (DA) foundation. Power differential allocation (PDA) additionally uses speed-dependent propeller-acceleration envelopes~\cite{cuniato2026}. A calibrated envelope is unavailable for the present propulsion model, so PDA is excluded from this benchmark.

A differential allocator uses twelve effective actuator coordinates because servo motion and propeller acceleration both contribute to wrench rate. Instantaneous feasibility instead concerns the six thrust channels at fixed angles. A direction available through coordinated motion may therefore be unavailable through thrust changes alone. The proposed barrier constrains this fixed-configuration set while retaining a differential tracking stage.

\subsection{Barrier Functions and Actuator Feasibility}
CBFs translate set invariance into control inequalities~\cite{ames2017,ames2019}. Robust, input-to-state-safe, and measurement-robust formulations address uncertainty in the dynamics, actuation, and state information~\cite{xu2015,kolathaya2019,cosner2021}. Integral and high-order CBFs accommodate dynamically extended inputs and higher-relative-degree constraints~\cite{ames2021,xiao2022}. Safety layers have also been studied for aerial physical interaction~\cite{cuniato2022} and reaction-wheel allocation~\cite{yang2025}.

Readiness barriers preserve a lower bound on motor wrench-rate authority through a torque-level allocation QP~\cite{silano2026}. More directly, Silano and Saska~\cite{silanosaska2026} use active tilt to preserve a motor-only wrench-rate readiness metric, enforced through motor torques and servo setpoints. Their certificate concerns a lower bound on the log-volume of an attainable wrench-rate approximation. Our constraint follows a specified wrench request: the request and its uncertainty neighborhood must fit inside the instantaneous thrust-limited wrench set. The square inverse map makes this containment test exact, and the equivalent-wrench identity transfers nominal margins to actual-angle feasibility. The distinction is the certified quantity: a desired-wrench neighborhood in the thrust-limited set, rather than a volume measure of motor wrench-rate authority.

The certificate is continuous time and conditional on admissible CBF inputs. Sampled-data and delayed-input formulations~\cite{breeden2022,singletary2020} address implementation effects, while command governors and predictive safety filters~\cite{garone2017,wabersich2021} can modify a request that exceeds the available authority. These are complementary extensions to the allocation-level result established here.

The present certificate concerns allocation feasibility; outer-loop stability and task-level safety remain separate properties.

\section{System Model and Allocation Objective}\label{sec:model}
\subsection{Rigid-Body Dynamics and Tilt Geometry}
Let $p,v\in\R^3$ denote inertial position and velocity, $R\in\SO$ the body-to-inertial rotation, and $\Omega\in\R^3$ the body angular velocity. With mass $m$, inertia $J_b\succ0$, and an upward inertial vertical axis,
\begin{align}
\dot p&=v,&m\dot v&=-mg e_3+RF+d_F,\label{eq:translation}\\
\dot R&=R\sk\Omega,&J_b\dot\Omega&=\tau-\Omega\times J_b\Omega+d_\tau.
\label{eq:rotation}
\end{align}
The allocator produces the body-frame wrench $w=[F^\top,\tau^\top]^\top$. The external force $d_F$ is expressed in inertial coordinates and the external moment $d_\tau$ in body coordinates. Throughout, $\sk x y=x\times y$, and $(\cdot)^\vee$ is the inverse of the hat map.

For thrust channel $i$, let $r_i$ be its body-frame position, $t_i$ a unit tangential direction orthogonal to $e_3$, $s_i\in\{-1,1\}$ its spin sign, and $\kappa_i$ its reaction-torque coefficient. Define
\begin{align}
d_i(\alpha_i)&=\cos\alpha_i\,e_3+\sin\alpha_i\,t_i,\label{eq:direction}\\
a_i(\alpha_i)&=\begin{bmatrix}d_i(\alpha_i)\\r_i\times d_i(\alpha_i)+s_i\kappa_i d_i(\alpha_i)\end{bmatrix},\label{eq:column}\\
w&=A(\alpha)f,\quad A(\alpha)=[a_1(\alpha_1)\ \cdots\ a_6(\alpha_6)].\label{eq:allocation}
\end{align}
The vectors $\alpha,f\in\R^6$ contain tilt angles and thrust magnitudes. Propeller speed $\omega_i$ is related to thrust by $f_i=k_{T,i}\omega_i^2$ for a positive thrust coefficient $k_{T,i}$. Identical co-tilting propellers with the same unit-thrust wrench column can be grouped into one channel by summing their thrusts and bounds. For a pair driven by a common speed command, $k_{T,i}$ includes both propellers. This reduction preserves the instantaneous attainable wrench set.

The allocator uses body-frame wrenches. Inertial force requests must therefore be rotated before applying~\eqref{eq:allocation}, and their derivatives must include the rotating-frame term given in Appendix~\ref{app:derivative}.

\subsection{Instantaneous Wrench Feasibility}
For constant thrust bounds $0\le\underline f<\overline f$, define
\begin{align}
\Uset&=\{f:\underline f\le f\le\overline f\},&
\Wset(\alpha)&=A(\alpha)\Uset.\label{eq:wrenchset}
\end{align}
On an open configuration domain $D\subset\R^6$ where $A(\alpha)$ is nonsingular, let $C(\alpha)=A(\alpha)^{-1}$. Then
\begin{equation}
w_d\in\Wset(\alpha)
\ \Longleftrightarrow\ 
\underline f\le C(\alpha)w_d\le\overline f.
\label{eq:exactfeasible}
\end{equation}
Although the reduced allocation model has twelve actuator coordinates, only the six channel thrusts remain free at fixed angles. Equation~\eqref{eq:exactfeasible} applies only where this square map remains nonsingular; all-vertical rotor directions, for example, cannot provide arbitrary six-dimensional instantaneous wrench.

On $D$, the thrust box maps to a full-dimensional parallelotope with twelve facets, each represented by a row of $C(\alpha)$ and a thrust bound. The inverse-map test gives every feasibility inequality without sampling wrench directions or constructing vertices. This exact reduction relies on the square nonsingular map; a redundant fixed-configuration map requires a different treatment of internal thrust freedom.

\subsection{Nominal and Actual Servo Configurations}
Let $\alpha_n$ be the nominal servo trajectory and $\alpha_r$ the actual configuration. Their difference is evaluated in consistently unwrapped angular coordinates. The proposed allocator chooses
\begin{equation}
\dot\alpha_n=\nu,\qquad \nu\in\Vset,
\label{eq:nominalpath}
\end{equation}
where $\Vset=\{\nu:\underline\nu\le\nu\le\overline\nu\}$ is the admissible nominal servo-rate set. Servo lag permits $\alpha_r\ne\alpha_n$ and $\dot\alpha_r\ne\nu$. Our objective is
\begin{equation}
w_d(t)\in\Wset(\alpha_r(t)).\label{eq:physicalgoal}
\end{equation}
This requires a bounded thrust solution at the actual angles. It does not assume that the actual thrust $f_r$ instantaneously equals that solution.

We distinguish the nominal required thrust $f_n=C(\alpha_n)w_d$, the actual-configuration requirement $C(\alpha_r)w_d$, the generated thrust $f_r$, and the command $f_c$. They enter the barrier, feasibility test, tracking feedback, and command integrator, respectively.

\section{From Geometric to Differential Allocation}\label{sec:baselines}
We express the baseline methods in the notation of Section~\ref{sec:model}, distinguishing actuator endpoint selection, differential tracking, and rate-limit enforcement. Table~\ref{tab:methods} summarizes their roles and the additional containment objective addressed by the proposed CBF-based control allocation (CA) method.

\begin{table*}[t]
\caption{Allocation methods and the scope of their constraints}
\label{tab:methods}
\centering\footnotesize
\renewcommand{\arraystretch}{1.2}
\begin{tabularx}{\textwidth}{@{}p{.13\textwidth}>{\raggedright\arraybackslash}p{.31\textwidth}>{\raggedright\arraybackslash}X@{}}
\toprule
Method & Allocation mechanism & Constraint or feasibility property\\
\midrule
GA~\cite{bodie2020} & Virtual-force inversion; thrust and tilt setpoint recovery & Selects an endpoint without constraining the wrench set along the servo transient.\\
DA~\cite{allenspach2020} & Weighted inverse with secondary null-space motion & Matches a differential wrench-rate target; imposes no explicit thrust-margin constraint.\\
ADA~\cite{cuniato2026} & Wrench-error feedback as the differential target & Uses actuator-state feedback for tracking; adds no wrench-containment constraint.\\
NDA~\cite{cuniato2026} & Rate normalization, common scaling, and first-order setpoints & Enforces the modeled rate box; scaling may sacrifice exact wrench-rate tracking.\\
PDA~\cite{cuniato2026} & Propulsion-dependent acceleration bounds & Incorporates speed-dependent power and acceleration limits; excluded from this benchmark.\\
Proposed\newline CBF-based CA & Servo-rate QP~\eqref{eq:qp} and differential thrust tracking & Preserves $w_d+\Eset\subseteq\Wset(\alpha_n)$; Theorem~\ref{thm:main} transfers its margins to the actual required thrust.\\
\bottomrule
\end{tabularx}
\par\smallskip
\begin{minipage}{\textwidth}\footnotesize
The proposed-method row states the full theoretical objective. The benchmark in Section~\ref{sec:simulation} uses its measured-angle prototype.
\end{minipage}
\end{table*}

\subsection{Geometric Allocation}
Write $a_i(\alpha_i)=a_{z,i}\cos\alpha_i+a_{t,i}\sin\alpha_i$, where $a_{z,i}$ and $a_{t,i}$ follow from~\eqref{eq:column} using $e_3$ and $t_i$, respectively. Introduce
\begin{align}
z_i&=\begin{bmatrix}f_i\cos\alpha_i&f_i\sin\alpha_i\end{bmatrix}^{\top},\nonumber\\
w&=\mathcal A z,\qquad
\mathcal A=[a_{z,1}\ a_{t,1}\ \cdots\ a_{z,6}\ a_{t,6}].
\end{align}
The basic minimum-norm GA solution~\cite{bodie2020} is
\begin{align}
z^*&=\mathcal A^\dagger w_d,\label{eq:ga}\\
f_i^*&=\sqrt{(z_{i,1}^*)^2+(z_{i,2}^*)^2},&
\alpha_i^*&=\operatorname{atan2}(z_{i,2}^*,z_{i,1}^*).\nonumber
\end{align}
This unconstrained solution gives a desired endpoint when $\mathcal A$ has full row rank; thrust and angle limits require separate handling. It does not constrain the wrench set encountered during servo motion. Angle recovery also needs a convention when $f_i^*$ approaches zero.

\subsection{Differential and Augmented Differential Allocation}
To coordinate actuator motion, DA~\cite{allenspach2020} uses $q=[\alpha^\top,\omega^\top]^\top$ and differentiates~\eqref{eq:allocation}:
\begin{align}
\dot w&=J(q)\dot q,\qquad J=[J_\alpha\ J_\omega],\label{eq:da_map}\\
(J_\alpha)_{:,i}&=a_i'(\alpha_i)f_i,\nonumber\\
J_\omega&=A(\alpha)\diag(2k_{T,i}\omega_i).\nonumber
\end{align}
Here $a_i'=\partial a_i/\partial\alpha_i$. For a selected wrench-rate target $v_w$, a weighted solution is
\begin{align}
\dot q&=J_W^\dagger v_w+(I-J_W^\dagger J)\dot q_0,\label{eq:da}\\
J_W^\dagger&=W^{-1}J^\top(JW^{-1}J^\top)^{-1},\quad W\succ0.\nonumber
\end{align}
The expression assumes full row rank of $J$. The null-space command $\dot q_0$ supports secondary objectives, while $W$ balances servo rates and propeller accelerations.

ADA~\cite{cuniato2026} supplies the wrench-controller interface
\begin{equation}
v_w=K_w\bigl(w_d-A(\alpha_r)f_r\bigr),\qquad K_w\succ0.
\label{eq:augmentation}
\end{equation}
Equation~\eqref{eq:da} then allocates an actuator-feedback wrench error without direct body-acceleration feedback. The virtual target $v_w$ is distinct from the total command derivative $\dot w_d$, which will enter the CBF.

The particular solution $J_W^\dagger v_w$ has minimum $W$-weighted norm. Null-space motion preserves $J\dot q=v_w$, but individual actuator saturation can destroy this equality.

For an ideal differential realization of ADA, the wrench error $e_w=w_d-w$ obeys
\begin{equation}
\dot e_w=\dot w_d-K_we_w.
\label{eq:adaerror}
\end{equation}
Thus~\eqref{eq:augmentation} gives exponential decay for a constant request. Feedforward $\dot w_d$ cancels the driving term for a changing request in the ideal model. Independently of that tracking choice, a barrier depending on $w_d$ requires its true total derivative.

\subsection{Normalized Differential Allocation}
NDA~\cite{cuniato2026} replaces manual rate weighting with actuator-limit normalization. For rate bounds $\underline{\dot q}<\overline{\dot q}$, define
\begin{equation}
\dot q_c=\tfrac12(\overline{\dot q}+\underline{\dot q}),\qquad
D_q=\tfrac12\diag(\overline{\dot q}-\underline{\dot q}).
\end{equation}
Using $\dot q=\dot q_c+D_qy$ and $J_N=JD_q$ gives
\begin{align}
y^*&=J_N^\dagger(v_w-J\dot q_c)+(I-J_N^\dagger J_N)y_0,\label{eq:nda}\\
\dot q&=\dot q_c+D_q\frac{y^*}{\max(1,\norm{y^*}_\infty)}.\label{eq:nda_sat}
\end{align}
For diagonal time constants $T_a\succ0$, the first-order model $\dot q=T_a^{-1}(q_{\mathrm{cmd}}-q)$ gives the setpoint $q_{\mathrm{cmd}}=q+T_a\dot q$.

With full-row-rank $J_N$ and $s=\max(1,\norm{y^*}_\infty)$,
\begin{equation*}
J\dot q=s^{-1}v_w+(1-s^{-1})J\dot q_c.
\end{equation*}
Thus $s=1$ gives exact allocation. Scaling enforces the rate box, yielding $v_w/s$ for symmetric bounds and an offset term for asymmetric bounds.

The CBF stage below adds instantaneous wrench-set constraints to rate-based tracking.

\section{CBF-Based Wrench-Feasibility Allocation}\label{sec:proposed}
The proposed allocator separates two functions (Fig.~\ref{fig:architecture}): the servo stage preserves wrench feasibility, and the thrust stage tracks the requested wrench. We construct a sufficient nominal containment condition for actual-angle feasibility, reduce it exactly to twelve thrust margins, and derive affine constraints on the servo rate.
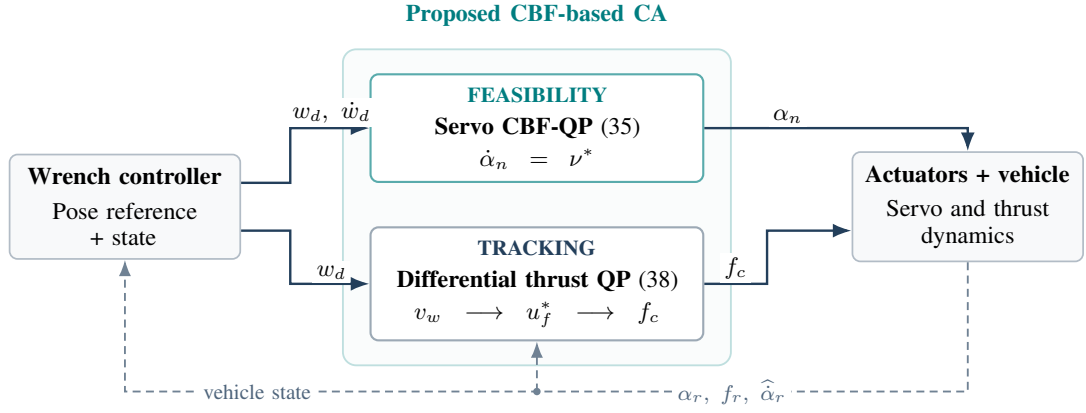
\begin{figure*}[t]
\centering
\begin{tikzpicture}[font=\small,>=Latex,
 outerbox/.style={draw=navy!35,line width=.65pt,rounded corners=3pt,
   align=center,fill=navy!3,text width=2.70cm,minimum height=1.30cm,inner sep=5pt},
 stage/.style={draw=teal!65,line width=.8pt,rounded corners=3pt,
   align=center,fill=white,text width=4.05cm,minimum height=1.28cm,inner sep=5pt},
 flow/.style={->,draw=navy,line width=.9pt},
 feedback/.style={->,draw=navy!65,line width=.7pt,dash pattern=on 3pt off 2pt},
 signal/.style={fill=white,inner sep=1.5pt,font=\small}]
\node[outerbox] (outer) at (0,0)
 {\textbf{Wrench controller}\\[3pt]Pose reference + state};
\node[stage] (servo) at (5.45,1.04)
 {\textcolor{teal}{\footnotesize\textbf{FEASIBILITY}}\\[2pt]
  \textbf{Servo CBF-QP}~\eqref{eq:qp}\\[2pt]
  $\dot\alpha_n=\nu^*$};
\node[stage,draw=navy!45] (thrust) at (5.45,-1.02)
 {\textcolor{navy}{\footnotesize\textbf{TRACKING}}\\[2pt]
  \textbf{Differential thrust QP}~\eqref{eq:thrustqp}\\[2pt]
  $v_w\;\longrightarrow\;u_f^*\;\longrightarrow\;f_c$};
\node[outerbox] (plant) at (11.15,0)
 {\textbf{Actuators + vehicle}\\[3pt]Servo and thrust dynamics};
\begin{scope}[on background layer]
\node[draw=teal!25,line width=.7pt,rounded corners=5pt,fill=teal!2,
 fit=(servo)(thrust),inner xsep=10pt,inner ysep=9pt] (allocator) {};
\end{scope}
\node[font=\small\bfseries,text=teal,above=5pt of allocator]
 {Proposed CBF-based CA};
\draw[flow] ([yshift=9pt]outer.east) -- ++(.65,0)
 |- node[pos=.76,above,signal] {$w_d,\ \dot w_d$} (servo.west);
\draw[flow] ([yshift=-9pt]outer.east) -- ++(.65,0)
 |- node[pos=.76,above,signal] {$w_d$} (thrust.west);
\draw[flow] (servo.east) -- node[above,signal] {$\alpha_n$}
 ++(2.21,0) -| (plant.north);
\draw[flow] (thrust.east) -- node[above,signal] {$f_c$}
 ++(.83,0) |- ([yshift=-9pt]plant.west);
\coordinate (feedbackbus) at ($(plant.south)+(0,-1.71)$);
\draw[feedback] (plant.south) -- (feedbackbus) -| (outer.south);
\draw[feedback] (thrust.south |- feedbackbus) -- (thrust.south);
\fill[navy!65] (thrust.south |- feedbackbus) circle (1.4pt);
\node[font=\footnotesize,text=navy!80,fill=white,inner sep=2pt] at ($(feedbackbus)+(-9.40,0)$) {vehicle state};
\node[font=\footnotesize,text=navy!80,fill=white,inner sep=2pt] at ($(feedbackbus)+(-3.12,0)$)
 {$\alpha_r,\ f_r,\ \widehat{\dot\alpha}_r$};
\end{tikzpicture}
\caption{Proposed CBF-based CA. The servo QP controls feasibility through the nominal trajectory $\alpha_n$; the thrust QP uses actual actuator feedback for tracking. Solid arrows denote commands and dashed arrows measured or estimated feedback. The derivative $\dot w_d$ enters the barrier constraints, while $v_w$ is formed from~\eqref{eq:augmentation}.}
\label{fig:architecture}
\end{figure*}

\subsection{Robust Allocation Objective}
The physical objective is $w_d\in\Wset(\alpha_r)$, whereas the servo decision acts on the nominal trajectory $\dot\alpha_n=\nu$. To connect these two configurations, let $A_n=A(\alpha_n)$, $A_r=A(\alpha_r)$, $C_n=A_n^{-1}$, and $C_r=A_r^{-1}$. Define the equivalent wrench perturbation induced by configuration mismatch
\begin{equation}
\delta w_s=(A_n-A_r)C_rw_d.
\label{eq:mismatch}
\end{equation}
It gives the exact nominal-to-actual relation
\begin{equation}
C_rw_d=C_n(w_d+\delta w_s).
\label{eq:transfer}
\end{equation}
The same vector $C_rw_d$ produces $w_d$ through $A_r$ and $w_d+\delta w_s$ through $A_n$. Equation~\eqref{eq:transfer} compares the two configurations without linearizing their angle mismatch or equating their servo velocities.

Choose a fixed ellipsoid that bounds this equivalent error over the operating domain:
\begin{equation}
\Eset=\{\delta:\delta^\top Q_w^{-1}\delta\le1\},\qquad Q_w\succ0.
\label{eq:ellipsoid}
\end{equation}
For example, $Q_w=\diag(r_F^2I_3,r_\tau^2I_3)$ sets separate force and torque radii. The robust nominal objective is
\begin{equation}
\boxed{\quad w_d(t)+\Eset\subseteq\Wset(\alpha_n(t)).\quad}
\label{eq:robustobjective}
\end{equation}
With $\delta w_s\in\Eset$, this containment makes $C_rw_d$ an admissible thrust vector by~\eqref{eq:transfer}. The servo stage must maintain~\eqref{eq:robustobjective} as $w_d$ changes, subject to the nominal rate limits.

The perturbation $\delta w_s$ captures configuration mismatch. The tracking error $w_d-A_rf_r$ also depends on motor output and supplies feedback to the thrust stage.

\subsection{Exact Reduction to Thrust Margins}
Let $c_i^\top$ denote row $i$ of $C_n$. Define the nominal required thrust and the reserve projected onto thrust channel $i$ as
\begin{equation}
f_n=C_nw_d,\qquad \rho_i=\sqrt{c_i^\top Q_wc_i}.
\label{eq:support}
\end{equation}
Writing $\delta=Q_w^{1/2}z$ with $\norm{z}_2\le1$ yields
\begin{equation}
\sup_{\delta\in\Eset}|c_i^\top\delta|
=\norm{Q_w^{1/2}c_i}_2=\rho_i.
\label{eq:supportmaximum}
\end{equation}
The equality follows from Cauchy--Schwarz, with the positive support attained at $\delta_i^\star=Q_wc_i/\rho_i$. Since $C_n$ is nonsingular and $Q_w\succ0$, every $\rho_i$ is strictly positive. Hence all perturbed wrenches $w_d+\delta$ can be produced within the thrust bounds exactly when
\begin{equation}
\underline f_i+\rho_i\le f_{n,i}\le\overline f_i-\rho_i,
\qquad i=1,\ldots,6.
\label{eq:tightenedinterval}
\end{equation}
The associated lower and upper margins are
\begin{equation}
B_i^-=f_{n,i}-\underline f_i-\rho_i,\qquad
B_i^+=\overline f_i-f_{n,i}-\rho_i.
\label{eq:barriers}
\end{equation}
Thus the complete six-dimensional containment is equivalent to twelve scalar inequalities:
\begin{equation}
B_i^\pm\ge0\quad\forall i
\quad\Longleftrightarrow\quad
w_d+\Eset\subseteq\Wset(\alpha_n).
\label{eq:containment}
\end{equation}
Each inequality bounds every point of the same ellipsoid. The worst-case perturbation may differ between channels; the support extrema need not be attained simultaneously.

\begin{figure}[!b]
\centering\includegraphics[width=.95\columnwidth]{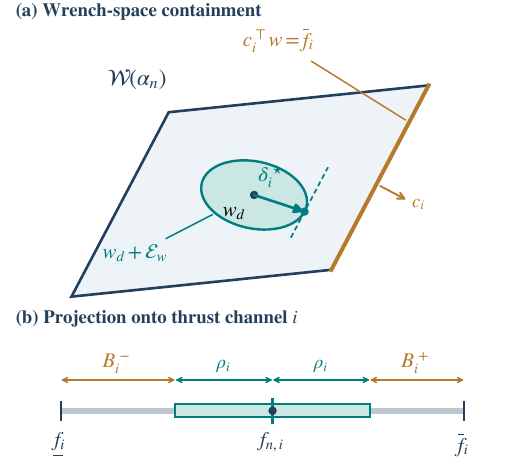}
\caption{Two-dimensional illustration of the exact reduction. (a) The neighborhood $w_d+\Eset$ lies inside $\Wset(\alpha_n)$; the support point satisfies $c_i^\top\delta_i^\star=\rho_i$. (b) Projection onto channel $i$ gives $[f_{n,i}-\rho_i,f_{n,i}+\rho_i]$. The remaining gaps to the thrust limits are $B_i^-$ and $B_i^+$.}
\label{fig:reserve}
\end{figure}
Figure~\ref{fig:reserve} connects the wrench-space construction to a scalar thrust interval. The uncertainty occupies $\rho_i$ on either side of $f_{n,i}$; the remaining gaps are the barriers $B_i^\pm$. Each $\rho_i$ has thrust units and varies with the inverse-map geometry even for fixed $Q_w$.

\subsection{Analytical Margin Dynamics}
Although containment depends nonlinearly on the servo angles, its margin derivatives are affine in the servo rate. From $C_nA_n=I$,
\begin{equation}
\frac{\partial C_n}{\partial\alpha_{n,j}}
=-C_n\frac{\partial A_n}{\partial\alpha_{n,j}}C_n.
\label{eq:inversegradient}
\end{equation}
Since only column $j$ of $A_n$ depends on $\alpha_{n,j}$, its derivative follows directly from $a_j'$. Define $M,P\in\R^{6\times6}$ and $b\in\R^6$ at $(\alpha_n,w_d)$ by
\begin{align}
\bar J_\alpha&=[a_1'(\alpha_{n,1})f_{n,1}\ \cdots\ a_6'(\alpha_{n,6})f_{n,6}],\nonumber\\
M&=-C_n\bar J_\alpha,\qquad b=C_n\dot w_d,\label{eq:Mb}\\
P_{ij}&=\frac{c_i^\top Q_w}{\rho_i}
\frac{\partial c_i}{\partial\alpha_{n,j}}.\label{eq:P}
\end{align}
Differentiating~\eqref{eq:support} then gives
\begin{align}
\dot f_n&=-C_n\dot A_nf_n+C_n\dot w_d=M\nu+b,\nonumber\\
\dot\rho_i&=\frac{c_i^\top Q_w\dot c_i}{\rho_i}=(P\nu)_i.
\label{eq:rhodot}
\end{align}
Stacking the six lower and upper margins as $B^-$ and $B^+$ gives
\begin{equation}
\dot B^-=(M-P)\nu+b,\qquad
\dot B^+=(-M-P)\nu-b.
\label{eq:bdot}
\end{equation}
These margin dynamics are exact; no finite-angle linearization is used.

A useful closed form makes the geometry dependence explicit. Define $D_\alpha=[a_1'\ \cdots\ a_6']$ at $\alpha_n$ and $G=C_nD_\alpha$. Equation~\eqref{eq:inversegradient} implies $\partial c_i/\partial\alpha_{n,j}=-G_{ij}c_j$, and hence
\begin{equation}
M_{ij}=-G_{ij}f_{n,j},\qquad
P_{ij}=-G_{ij}\frac{c_i^\top Q_wc_j}{\rho_i}.
\label{eq:closedcoefficients}
\end{equation}
The coefficient $G_{ij}$ couples servo $j$ to both required thrust and reserve in channel $i$. A motion can reduce $f_{n,i}$ while increasing $\rho_i$, so differentiating required thrust alone is insufficient.

The term $b=C_n\dot w_d$ accounts for motion of the requested wrench. Substituting $v_w$ for $\dot w_d$ would omit this motion when $w=w_d$, because the feedback target then vanishes. Appendix~\ref{app:derivative} derives $\dot w_d$ for representative translational and attitude controllers, including frame rotation and compensation terms.

\subsection{Servo-Rate Quadratic Program}
With the linear extended class-$\mathcal K$ choice $\gamma B_i^\pm$, $\gamma>0$, the required barrier inequalities are
\begin{equation}
\dot B_i^\pm+\gamma B_i^\pm\ge0,\qquad i=1,\ldots,6.
\label{eq:cbfcondition}
\end{equation}
These inequalities yield the complete servo QP:
\begin{equation}
\boxed{
\begin{aligned}
\nu^*=\underset{\nu\in\R^6}{\arg\min}\quad &
\tfrac12\nu^\top R_\nu\nu+
\tfrac12\norm{M\nu+b}_{R_f}^{2}\\[2pt]
\text{subject to}\quad &
(M-P)\nu+b+\gamma B^-\ge0,\\
&(-M-P)\nu-b+\gamma B^+\ge0,\\
&\underline\nu\le\nu\le\overline\nu.
\end{aligned}}
\label{eq:qp}
\end{equation}
Here $R_\nu\succ0$, $R_f\succeq0$, and $\norm{x}_{R_f}^2=x^\top R_fx$. All vector inequalities are componentwise. The twelve barrier constraints preserve the tightened thrust intervals; the rate box bounds $\dot\alpha_n$, not the actual rate $\dot\alpha_r$. The nominal command is obtained by integrating $\dot\alpha_n=\nu^*$. This trajectory is retained as the state of the certificate; replacing it by the measured angle at each update defines different barrier dynamics.

In~\eqref{eq:qp}, $b$, $M\nu$, and $P\nu$ represent requested-wrench motion, reorientation, and reserve variation, respectively.

The cost penalizes servo motion and nominal required-thrust rate. The latter regularizes reconfiguration without imposing a hard motor-acceleration limit. The Hessian $H=R_\nu+M^\top R_fM\succ0$ gives a unique minimizer whenever the feasible set is nonempty. With $R_f=0$, the QP selects the feasible servo rate with minimum $R_\nu$-weighted norm and leaves the servos stationary when zero rate is admissible.

\subsection{Boundary Behavior}
The two barrier inequalities for channel $i$ can be written as a single interval for the nominal required-thrust rate:
\begin{equation}
\dot\rho_i-\gamma B_i^-
\ \le\ \dot f_{n,i}\ \le\
-\dot\rho_i+\gamma B_i^+.
\label{eq:ratecorridor}
\end{equation}
At the lower boundary, $B_i^-=0$ requires $\dot f_{n,i}\ge\dot\rho_i$: required thrust must keep pace with a growing reserve. At the upper boundary, $B_i^+=0$ requires $\dot f_{n,i}\le-\dot\rho_i$: required thrust must decrease as the reserve grows. Holding $f_{n,i}$ constant is insufficient when reconfiguration consumes additional reserve.

Positive margins relax the interval endpoints by $\gamma B_i^\pm$. Larger $\gamma$ permits faster approach to a boundary, while $Q_w$ sets the reserved authority. Changing $\gamma$ neither enlarges the thrust box nor changes uncertainty coverage. At the boundary, the inward-pointing derivative condition is independent of $\gamma$.

A common $\nu\in\Vset$ must satisfy all six intervals; separate interval feasibility is insufficient. The QP enforces these coupled inequalities jointly. For constant $w_d$, nonnegative margins, and a rate box containing zero, $\nu=0$ is feasible and minimizes the cost. The nominal geometry can remain fixed while motor feedback addresses wrench-tracking error.

\subsection{Differential Thrust Allocation}
At the actual actuator state, the wrench-rate identity is
\begin{equation}
\dot w=J_\alpha(\alpha_r,f_r)\dot\alpha_r+A_r\dot f_r.
\label{eq:wdot}
\end{equation}
The thrust stage uses a measured or estimated actual servo rate $\widehat{\dot\alpha}_r$ and the ADA target~\eqref{eq:augmentation}. Let $u_f$ denote the thrust-command rate. Using it as a surrogate for $\dot f_r$ in the differential tracking model gives
\begin{equation}
\begin{aligned}
u_f^*=\underset{u_f}{\arg\min}\quad &
\tfrac12\norm{A_ru_f+J_\alpha\widehat{\dot\alpha}_r-v_w}_{Q_a}^2\\
&{}+\tfrac\epsilon2\norm{u_f-u_{f,0}}^2\\
\text{subject to}\quad &\underline u_{f,k}\le u_f\le\overline u_{f,k}.
\end{aligned}
\label{eq:thrustqp}
\end{equation}
Here $Q_a\succ0$, $\epsilon>0$, and $u_{f,0}$ is a chosen rate reference. The rate bounds incorporate the thrust-command interval:
\begin{align}
\underline u_{f,k}&=\max\left(\underline u_f,\frac{\underline f-f_{c,k}}{\Delta t}\right),\nonumber\\
\overline u_{f,k}&=\min\left(\overline u_f,\frac{\overline f-f_{c,k}}{\Delta t}\right),\label{eq:thrustbounds}\\
f_{c,k+1}&=f_{c,k}+\Delta t\,u_f^*.\nonumber
\end{align}
Here $k$ indexes command updates, $\Delta t>0$ is their period, and the maximum and minimum act componentwise. For $f_{c,0}\in\Uset$ and $\underline u_f\le0\le\overline u_f$, zero rate remains admissible and~\eqref{eq:thrustbounds} keeps $f_c$ inside $\Uset$.

Without active bounds or regularization, the solution reduces to $u_f=C_r(v_w-J_\alpha\widehat{\dot\alpha}_r)$. The objective uses a command-rate approximation; motor dynamics determine how $f_r$ follows $f_c$, so $u_f=\dot f_r$ is not an assumed physical identity. A feedforward term $\dot w_d$ may be added to $v_w$ for tracking, independently of its mandatory role in the barrier derivative.

The measured or estimated servo-rate term accounts for wrench change already caused by reorientation. Thus, thrust tracking uses actual actuator feedback, while the nominal trajectory defines the robust certificate proved next.

\section{Wrench-Feasibility Guarantee}\label{sec:guarantee}
\subsection{Conditional Forward Invariance}
The certificate set is $\mathcal C=\{(\alpha_n,w_d)\in D\times\R^6:B^-\ge0,\ B^+\ge0\}$. Let $\alpha_n$, $\alpha_r$, and $w_d$ be locally absolutely continuous on $I=[t_0,T)$, with both angle trajectories in $D$ and $\dot\alpha_n=\nu$ almost everywhere. Assume $A\in C^1(D;\R^{6\times6})$ is nonsingular.
\begin{samepage}
\begin{theorem}\label{thm:main}
Suppose (i) $B_i^\pm(t_0)\ge0$ for all $i$; (ii) $\delta w_s(t)\in\Eset$ on $I$; and (iii) the selected locally integrable command satisfies the hard constraints of~\eqref{eq:qp} almost everywhere, using the total derivative $\dot w_d$. Then
\begin{align}
B_i^\pm(t)&\ge e^{-\gamma(t-t_0)}B_i^\pm(t_0)\ge0,\label{eq:invariance}\\
\underline f+B^-(t)&\le C_r(t)w_d(t)\le\overline f-B^+(t).\label{eq:actualguarantee}
\end{align}
For all $t\in I$, $(\alpha_n,w_d)$ remains in $\mathcal C$, and each required-thrust margin at the actual angles is at least its corresponding nominal barrier margin.
\end{theorem}
\end{samepage}
\begin{proof}
The stated regularity and $Q_w\succ0$ make the margins locally absolutely continuous. Equations~\eqref{eq:qp} and~\eqref{eq:bdot} imply $\dot B_i^\pm+\gamma B_i^\pm\ge0$ almost everywhere. Multiplication by $e^{\gamma(t-t_0)}$ and integration yield~\eqref{eq:invariance}.

Since $\delta w_s\in\Eset$, the ellipsoid support gives $|c_i^\top\delta w_s|\le\rho_i$. By~\eqref{eq:transfer}, the actual required thrust satisfies
\begin{equation}
f_{n,i}-\rho_i\le (C_rw_d)_i\le f_{n,i}+\rho_i.
\end{equation}
Substituting $f_{n,i}-\rho_i=\underline f_i+B_i^-$ and $f_{n,i}+\rho_i=\overline f_i-B_i^+$ gives~\eqref{eq:actualguarantee}. Nonnegative margins therefore ensure actual-configuration feasibility.
\end{proof}
Bounded mismatch thus transfers nominal containment to the actual configuration without equating nominal and actual servo rates. The guarantee holds along an admissible solution of~\eqref{eq:qp}; recursive QP feasibility and instantaneous thrust realization are not conclusions of the theorem.

\subsection{Reserve Selection and Conditioning}
The uncertainty reserve must bound~\eqref{eq:mismatch} over the claimed operating domain. For a nonsingular wrench scaling $S_w$, suppose
\begin{align}
\norm{S_w(A_n-A_r)}_2&\le L_A\norm{\alpha_n-\alpha_r}_2,\nonumber\\
\sigma_{\min}(S_wA_r)&\ge\underline\sigma>0.
\end{align}
The exact factorization
\begin{equation*}
S_w\delta w_s=S_w(A_n-A_r)(S_wA_r)^{-1}S_ww_d
\end{equation*}
and submultiplicativity of the induced norm give
\begin{equation}
\norm{S_w\delta w_s}_2\le
\frac{L_A}{\underline\sigma}\norm{\alpha_n-\alpha_r}_2\norm{S_ww_d}_2.
\label{eq:boundconstruction}
\end{equation}
For bounds $\norm{\alpha_n-\alpha_r}_2\le\bar e_\alpha$ and $\norm{S_ww_d}_2\le\bar w$, a sufficient fixed reserve is
\begin{equation*}
Q_w=\beta^2S_w^{-1}S_w^{-\top},\qquad
\beta\ge\frac{L_A\bar e_\alpha\bar w}{\underline\sigma},\quad\beta>0.
\end{equation*}
The corresponding ellipsoid is the ball $\norm{S_w\delta}_2\le\beta$ in scaled wrench coordinates. Thus the reserve accounts for configuration mismatch, wrench magnitude, and actual-map conditioning.

A larger reserve is not always admissible. Adding the two margins gives the necessary condition
\begin{equation}
2\rho_i\le\overline f_i-\underline f_i.\label{eq:widthcondition}
\end{equation}
For the square invertible map, it is also sufficient for a nonempty tightened wrench set at a fixed configuration: choose $f_n$ at the midpoint of each thrust interval. It does not ensure feasibility for a prescribed $w_d$ or its rate. Whenever~\eqref{eq:widthcondition} holds for every channel, $\Delta f=\overline f-\underline f$ gives
\begin{equation}
\sigma_{\min}(Q_w^{-1/2}A_n)\ge\frac{2}{\norm{\Delta f}_2}.\label{eq:conditionbound}
\end{equation}
This follows from
\begin{equation*}
\begin{aligned}
\norm{C_nQ_w^{1/2}}_2
&\le\norm{C_nQ_w^{1/2}}_F\\
&=\Bigl(\sum_i\rho_i^2\Bigr)^{1/2}
\le\tfrac12\norm{\Delta f}_2.
\end{aligned}
\end{equation*}
Nonnegative robust margins thus impose a nominal conditioning bound; actual-map conditioning remains part of the operating-domain assumptions.

\section{Simulation Study}\label{sec:simulation}
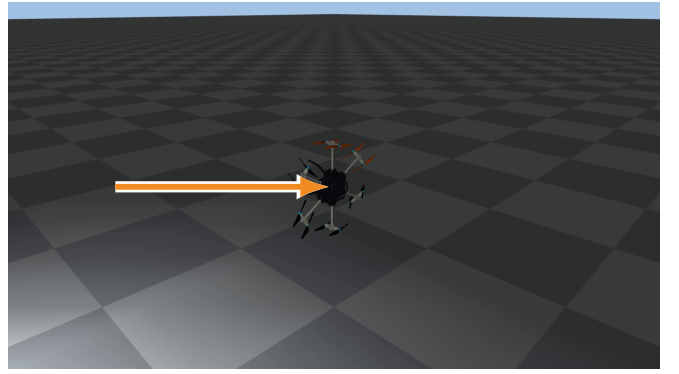
\begin{figure*}[t]
\centering
\begin{tabular}{@{}cc@{}}
\multicolumn{1}{@{}l}{\small\color{navy}\textbf{(a) S1}\quad Circular tracking}&
\multicolumn{1}{l@{}}{\small\color{navy}\textbf{(b) S2}\quad Circle with disturbance}\\[2pt]
\multicolumn{1}{@{}l}{\footnotesize\color{navy!70}No external force}&
\multicolumn{1}{l@{}}{\footnotesize\color{disturbance!80!black}$d_F=[5,-3,4]^\top\sin(2\pi t)\ \mathrm{N}$}\\[4pt]
\includegraphics[width=.475\textwidth]{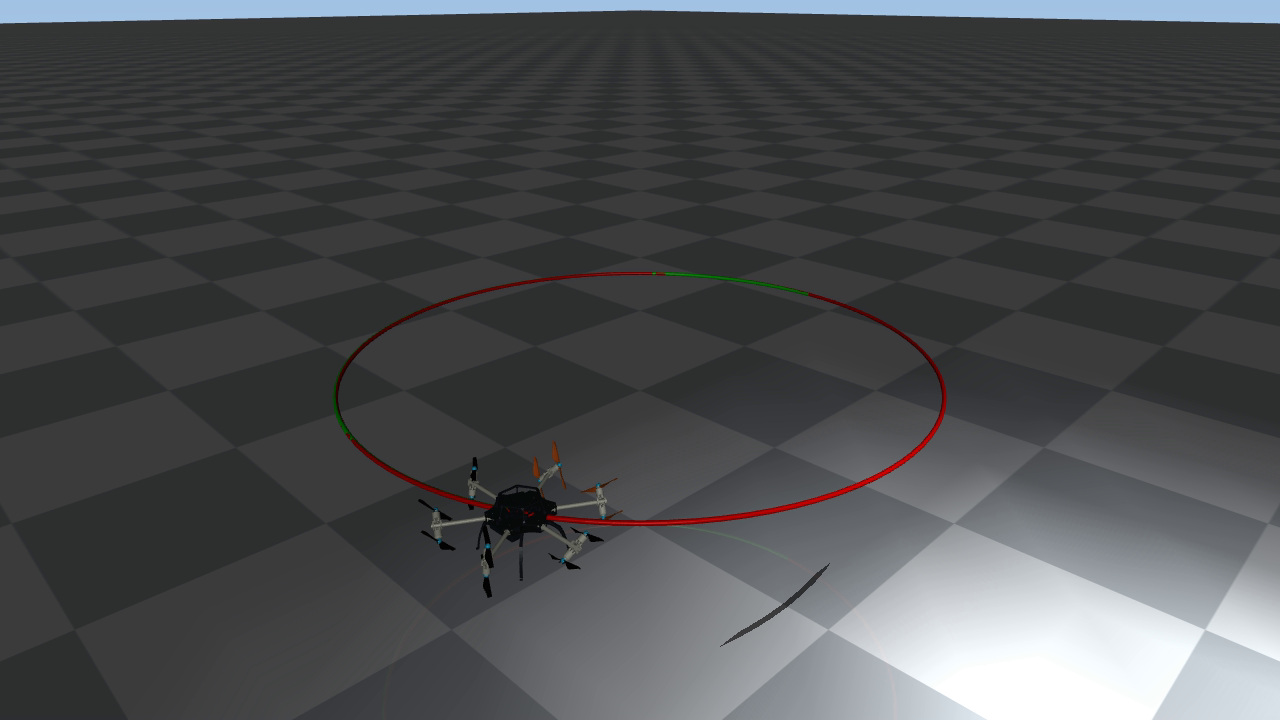}&
\begin{tikzpicture}[baseline=(scene.south)]
\node[anchor=south west,inner sep=0] (scene) at (0,0)
{\includegraphics[width=.475\textwidth]{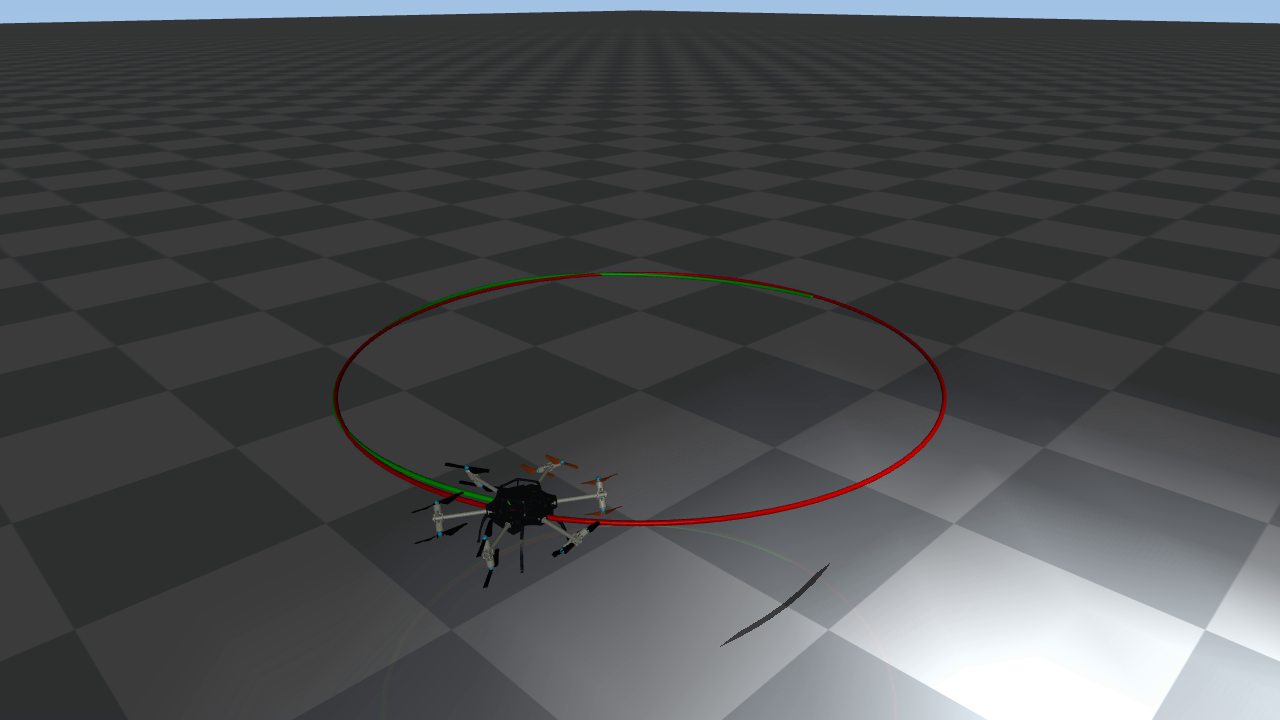}};
\begin{scope}[x={(scene.south east)},y={(scene.north west)}]
\draw[draw=white,line width=4.2pt,-{Latex[length=4.5mm,width=3.4mm]}]
(.180,.075) -- (.405,.300);
\draw[draw=disturbance,line width=2.2pt,-{Latex[length=4mm,width=2.8mm]}]
(.180,.075) -- (.405,.300);
\end{scope}
\end{tikzpicture}\\[-1pt]
\noalign{\vskip 7pt}
\multicolumn{1}{@{}l}{\small\color{navy}\textbf{(c) S3}\quad Cartwheel}&
\multicolumn{1}{l@{}}{\small\color{navy}\textbf{(d) S4}\quad Cartwheel with disturbance}\\[2pt]
\multicolumn{1}{@{}l}{\footnotesize\color{navy!70}No external force}&
\multicolumn{1}{l@{}}{\footnotesize\color{disturbance!80!black}$d_F=[6,0,0]^\top\sin(2\pi t)\ \mathrm{N}$}\\[4pt]
\includegraphics[width=.475\textwidth]{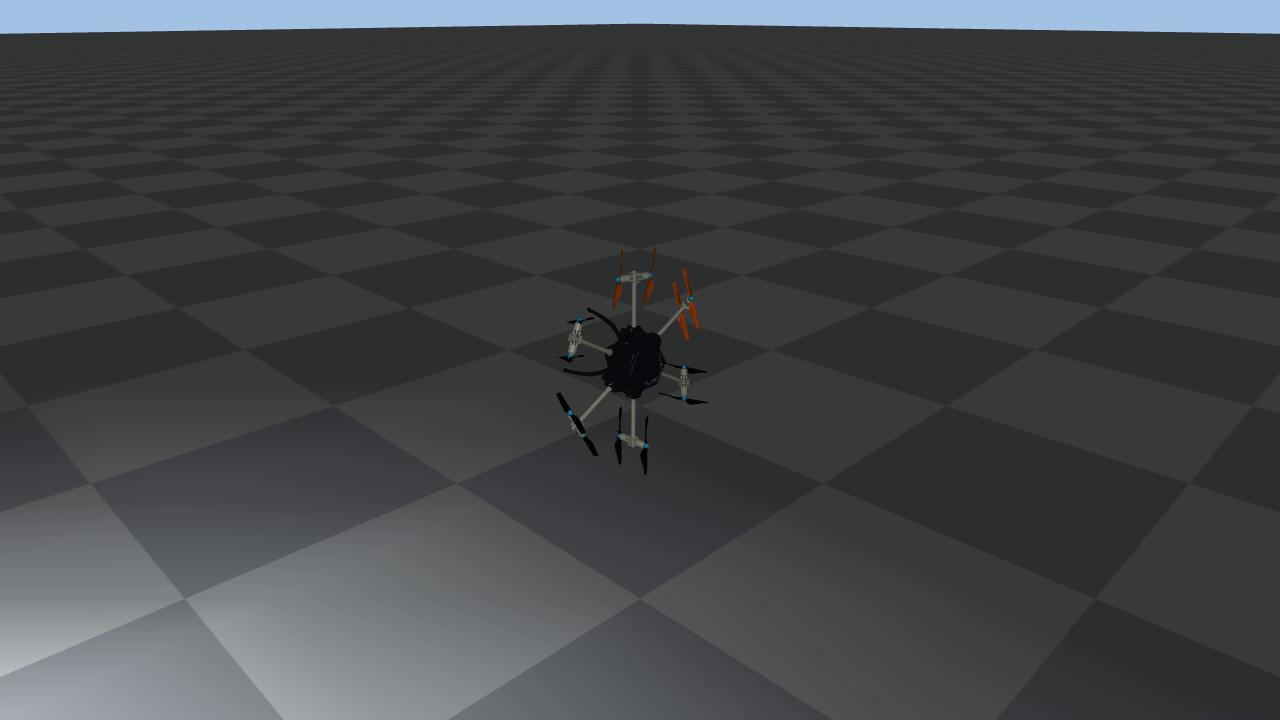}&
\begin{tikzpicture}[baseline=(scene.south)]
\node[anchor=south west,inner sep=0] (scene) at (0,0)
{\includegraphics[width=.475\textwidth]{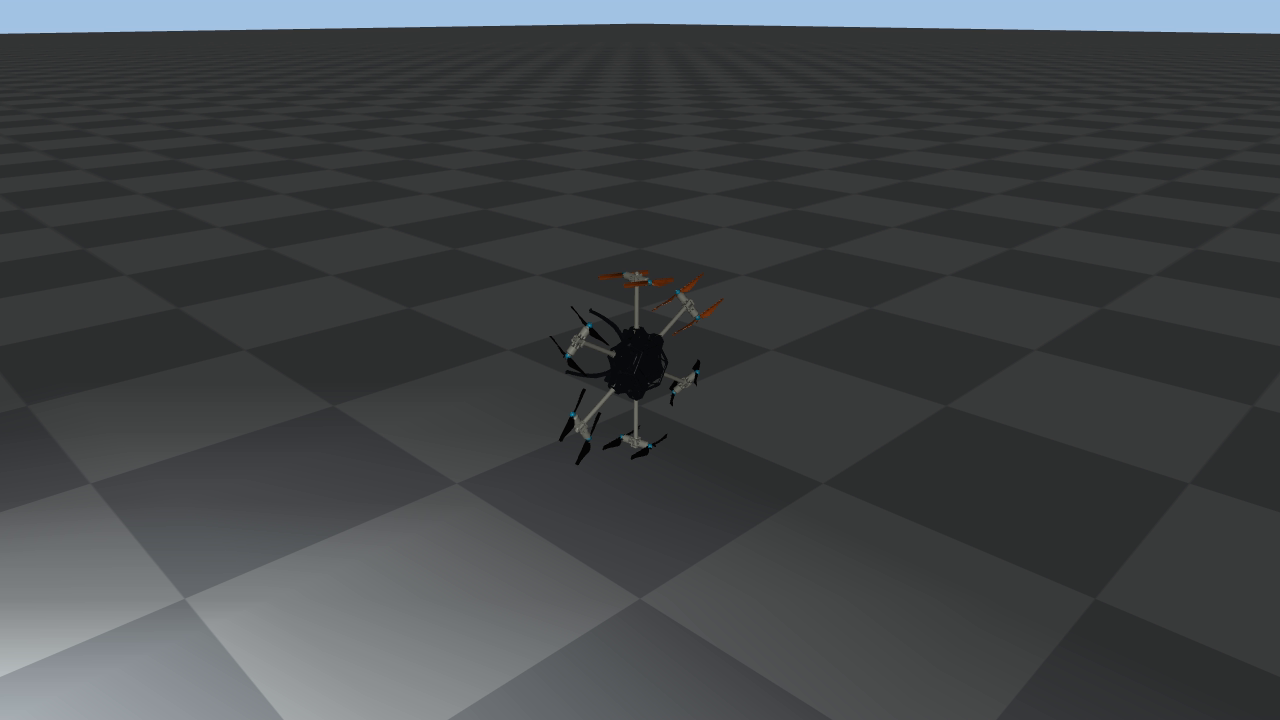}};
\begin{scope}[x={(scene.south east)},y={(scene.north west)}]
\draw[draw=white,line width=4.2pt,-{Latex[length=4.5mm,width=3.4mm]}]
(.165,.496) -- (.495,.496);
\draw[draw=disturbance,line width=2.2pt,-{Latex[length=4mm,width=2.8mm]}]
(.165,.496) -- (.495,.496);
\end{scope}
\end{tikzpicture}
\end{tabular}

\caption{Recorded MuJoCo scenes for the measured-angle implementation of the proposed CBF-based CA. Scenario labels specify world-frame external forces. Orange arrows point toward the body and indicate the positive half-cycle schematically; their directions reverse when $\sin(2\pi t)<0$.}
\label{fig:scenes}
\end{figure*}

\subsection{Setup and Scenarios}
The study uses an actively tilting vehicle with six servo joints and twelve propellers in MuJoCo~\cite{todorov2012}. Each pair of propellers shares a tilt angle and an identical unit-thrust wrench column, so its two thrusts form one effective channel in~\eqref{eq:allocation}. The recorded per-propeller command range is $0$--$8$ N, giving $0\le f_i\le16$ N for each pair. The setup includes servo and thrust dynamics and a PID outer controller with a nonlinear robust internal-loop compensator (NRIC)~\cite{kim2015}. The explicitly recorded outer-controller settings are common to all four methods and are summarized in Table~\ref{tab:settings}.

Results labeled Proposed CBF-based CA use the measured-angle prototype of the proposed allocator. The formulation in Sections~\ref{sec:proposed}--\ref{sec:guarantee} additionally requires a persistent nominal trajectory and an explicit configuration-mismatch reserve. Accordingly, this study evaluates tracking and actual-angle wrench feasibility; direct validation of the nominal-trajectory certificate requires further records.

\begin{table}[t]
\caption{Recorded simulation scenarios}
\label{tab:scenarios}
\centering\footnotesize
\renewcommand{\arraystretch}{1.18}
\begin{tabularx}{\columnwidth}{@{}l>{\raggedright\arraybackslash}X>{\raggedright\arraybackslash}X@{}}
\toprule
Case & Reference motion & World-frame force\\
\midrule
S1 & Circle: $r=1.5$ m, $T_c=3.0$ s & None\\
S2 & Same circle & $[5,-3,4]^\top\sin(2\pi t)$ N\\
S3 & Cartwheel: $90^\circ$ pitch, then $2.1$ rad/s body yaw & None\\
S4 & Same cartwheel & $[6,0,0]^\top\sin(2\pi t)$ N\\
\bottomrule
\end{tabularx}
\end{table}

\begin{table}[t]
\caption{Common settings preserved in the run records}
\label{tab:settings}
\centering\footnotesize
\renewcommand{\arraystretch}{1.12}
\begin{tabularx}{\columnwidth}{@{}l>{\raggedright\arraybackslash}X@{}}
\toprule
Quantity & Recorded value\\
\midrule
Actuator groups & 6 tilt joints, 12 propellers\\
Thrust-command limits & $0$--$8$ N per propeller; $0$--$16$ N per pair\\
Attitude gains & $K_R=\diag(10,10,6)$;\newline $K_\Omega=\diag(2.8,2.8,1.8)$;\newline integral gain zero\\
NRIC gain settings & $(k_p,k_d,k_i)=(300,30,3)$\\
NRIC filter time constant & $0.10$ s\\
Run length / logging interval & $20$ s / $8$ ms (2500 samples)\\
Cartwheel servo position limit & Disabled for continuous rotation\\
\bottomrule
\end{tabularx}
\end{table}

\begin{table}[!tb]
\caption{Tracking performance across the four scenarios}
\label{tab:results}
\centering\small
\renewcommand{\arraystretch}{1.12}
\begin{tabular}{@{}lrrrr@{}}
\toprule
Case & GA & ADA & NDA & \shortstack{Proposed\\CBF-based CA}\\
\midrule
\multicolumn{5}{c}{Position RMSE $E_p$ (m)}\\
S1 & 0.081 & 0.028 & 0.023 & \textbf{0.006}\\
S2 & 30.474 & 0.057 & 0.054 & \textbf{0.026}\\
S3 & 8.601 & 0.070 & 0.084 & \textbf{0.025}\\
S4 & 16.024 & 1.045 & 0.786 & \textbf{0.028}\\
\midrule
\multicolumn{5}{c}{Attitude RMSE $E_R$ ($^\circ$)}\\
S1 & 0.06 & \textbf{0.04} & \textbf{0.04} & 0.21\\
S2 & 128.27 & 0.09 & \textbf{0.07} & 0.33\\
S3 & 101.72 & 1.53 & 2.01 & \textbf{1.15}\\
S4 & 116.57 & 50.01 & 30.37 & \textbf{1.16}\\
\bottomrule
\end{tabular}
\par\smallskip
\begin{minipage}{.97\columnwidth}\footnotesize
Each entry uses all 2500 stored samples over 20 s, including the initial transient. One run is available per method and scenario. Lowest displayed values are bold.
\end{minipage}
\end{table}

Table~\ref{tab:scenarios} defines the four scenarios, and Fig.~\ref{fig:scenes} shows their representative scenes. For the cartwheel, the command specifies a 2 s takeoff, 1 s hold, 2 s pitch transition, and 1 s post-pitch hold. Continuous body yaw begins at $t=6$ s. Servo position limits are disabled for this continuous rotation. The circular reference has speed $3.14$ m/s and acceleration $6.58$ m/s$^2$ on its constant-speed segment. Each run contains 20 s of telemetry at 8 ms intervals; this is the logging interval, not the controller period. The preserved records support reproducible telemetry analysis but lack complete allocator gains, null-space targets, and actuator-model settings for a full simulation rerun.

Both maneuvers vary the body-frame wrench demand. The 1 Hz external forces in S2 and S4 act through the vehicle dynamics and feedback command; they do not define the configuration-mismatch ellipsoid in~\eqref{eq:ellipsoid}.

\subsection{Tracking Comparison}
Let $e_{p,k}$ denote the position error and $e_{\log,k}$ the recorded $\SO$ logarithmic attitude-error vector in radians at sample $k$. Table~\ref{tab:results} reports
\begin{equation}
\begin{aligned}
E_p&=\left(\frac{1}{N}\sum_{k=1}^N\norm{e_{p,k}}_2^2\right)^{1/2},\\
E_R&=\frac{180}{\pi}\left(\frac{1}{N}\sum_{k=1}^N\norm{e_{\log,k}}_2^2\right)^{1/2},
\end{aligned}
\label{eq:rmse}
\end{equation}
with $N=2500$, $E_p$ in meters, and $E_R$ in degrees. The logarithmic error used for evaluation differs from the geometric controller error $e_R$ in Appendix~\ref{app:derivative}. Recalculation from all stored samples reproduces the run summaries.

In S1 and S2, the proposed allocator achieves position RMSEs of $0.006$ m and $0.026$ m, respectively. ADA and NDA track both circles, whereas GA loses tracking in S2. Relative to ADA and NDA, the proposed method has higher attitude RMSE in both cases; these runs do not isolate how allocator settings affect this position--attitude tradeoff.

In the undisturbed cartwheel, the proposed allocator improves both position and attitude tracking relative to ADA and NDA. Under disturbance, its errors remain close to the undisturbed case: $0.028$ m and $1.16^\circ$ in S4, compared with $0.025$ m and $1.15^\circ$ in S3. The S4 attitude errors of ADA and NDA rise to $50.01^\circ$ and $30.37^\circ$, respectively. We next compare desired-wrench feasibility at the measured servo configurations.

\subsection{Disturbed-Cartwheel Wrench Feasibility}
Measured servo angles reconstruct $A_r(t)$, giving the required pair thrust $f^{\rm req}=C_rw_d$. The geometry, recovered from the recorded thrust--wrench relation, reproduces all sixteen runs' generated wrenches within $10^{-6}$ N per force component and $10^{-6}$ N\,m per torque component. This check uses actual thrust independently of the desired-wrench feasibility test.

Let $c_{r,i}^\top$ be row $i$ of $C_r$ and $D_\ell=\diag(I_3,\ell I_3)$, with $\ell=0.30$ m for all methods. Define the signed facet margin
\begin{equation}
\begin{aligned}
m_i(t)&=\min\{f_i^{\rm req}-\underline f_i,\;\overline f_i-f_i^{\rm req}\},\\
\mu_w(t)&=\min_i\frac{m_i(t)}{\norm{c_{r,i}^{\top}D_\ell}_2}.
\end{aligned}
\label{eq:loggedmargin}
\end{equation}
The coordinates $D_\ell^{-1}w$ express force and moment in force-equivalent units. The normalization preserves the membership test: $\mu_w\ge0$ if and only if $w_d\in\Wset(\alpha_r)$. For a feasible request, $\mu_w$ is the radius of the largest Euclidean ball centered at $D_\ell^{-1}w_d$ inside the scaled wrench set. Its magnitude depends on the chosen force--moment scaling, while its sign does not. Negative values indicate a violated facet; their magnitude is not the distance to the set.

\begin{figure*}[!p]
\centering
\includegraphics[width=\textwidth]{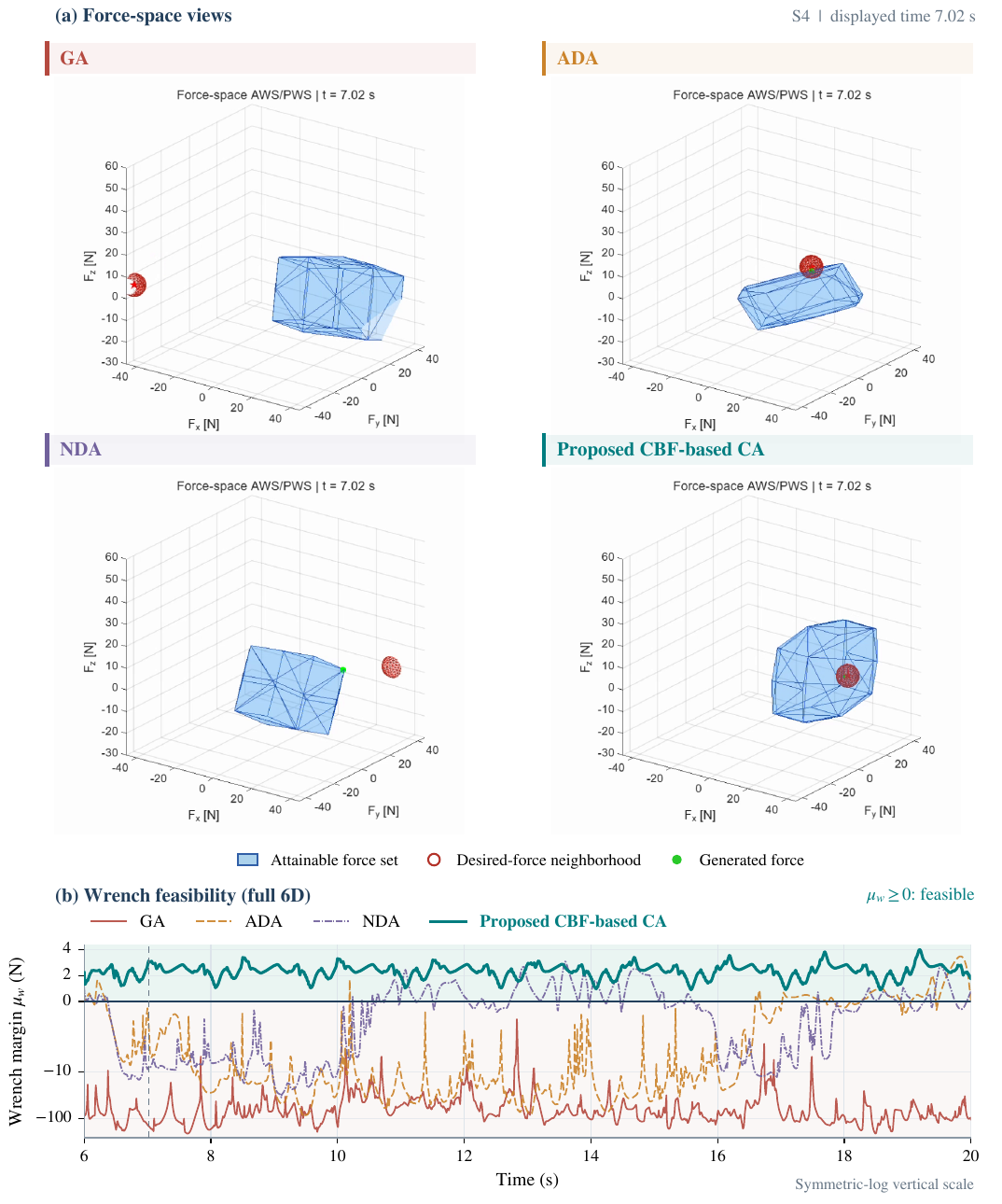}
\caption{Disturbed cartwheel (S4). (a) Original video frames at the displayed time $7.02$ s, including the measured-angle implementation of the proposed CBF-based CA. The legend identifies the native force-space objects; the red neighborhood is not a certificate of the prescribed $\Eset$. (b) Full six-dimensional margin~\eqref{eq:loggedmargin} over continuous rotation: nonnegative values indicate feasibility. The vertical dashed line marks the displayed frame time. All samples in this interval are plotted without smoothing on a symmetric-log scale.}
\label{fig:wrench_feasibility}
\end{figure*} 

Figure~\ref{fig:wrench_feasibility}(a) compares all four methods at the same displayed time. The baseline views place the desired-force neighborhood outside or across the attainable-set boundary; the proposed allocator's request lies inside. These force-space views are qualitative. Quantitative membership is evaluated in panel~(b) using the full force--moment request.

Over $6\le t\le20$ s, all 1751 proposed-method samples are feasible ($\min\mu_w=0.89$ N; minimum pair-thrust margin $0.96$ N). Violation rates are $100\%$ (GA), $80.7\%$ (ADA), and $60.8\%$ (NDA). These are sampled membership results, not an intersample guarantee.

\clearpage
\section{Discussion}\label{sec:discussion}
\subsection{Geometric Feasibility and Actuator Response}
The sampled test in Section~\ref{sec:simulation} checks whether bounded thrusts can realize the request at the measured configuration. Motor response and estimator errors still determine tracking accuracy. An attainable wrench can therefore coexist with transient error, whereas thrust redistribution alone cannot repair a wrench-set violation at fixed angles.

For a chosen ellipsoid and square nonsingular map, the support-based tightening is exact. Conservatism enters through the uncertainty bound: a domain-wide reserve may exceed the mismatch in a particular maneuver. A smaller reserve frees thrust authority but must still cover the equivalent configuration error; tracking quality alone does not establish that coverage.

Servo motion supplies differential allocation freedom, while the fixed-configuration map remains square. For more than six independent thrust channels, the containment test must account for internal thrust redundancy; substituting a pseudoinverse into the same inequalities does not preserve their equivalence.

\subsection{Design and Validation Priorities}
The servo-rate box bounds the nominal command rate, while the barrier inequalities constrain its effect on the moving wrench set. The weights $R_\nu,R_f$ select among feasible motions; the reserve and the boundary-approach gain determine the constraints described in Section~\ref{sec:proposed}. Tuning the cost cannot compensate for a request whose barrier constraints admit no servo motion.

A direct validation of the nominal formulation should record its integrated $\alpha_n$, measured $\alpha_r$, and prescribed $Q_w$ alongside the wrench. The minimum tightened margin and the coverage measure $\delta w_s^\top Q_w^{-1}\delta w_s$ then test nominal containment and reserve coverage at the stored samples. Substituting measured angles for $\alpha_n$ would erase the mismatch being tested. Repeated runs, identified actuator dynamics, and hardware trials can subsequently establish the operating envelope.

\section{Conclusion}\label{sec:conclusion}
This paper developed a servo-rate CBF allocator for preserving desired-wrench feasibility during tiltrotor reconfiguration. An exact nominal-to-actual required-thrust identity and ellipsoid support reduce robust containment to twelve margins. Their complete derivatives yield a six-variable QP and a conditional forward-invariance guarantee, with differential thrust allocation supplying the tracking stage. A complementary measured-angle CBF study maintains feasibility at every stored continuous-rotation sample of a disturbed cartwheel and achieves lower tracking error than GA, ADA, and NDA in that case. Circular tracking exhibits an attitude tradeoff. Future validation will assess the persistent nominal trajectory and its mismatch reserve, followed by hardware evaluation with identified actuator dynamics.

\newpage
\appendices
\section{Analytical Desired-Wrench Derivatives}\label{app:derivative}
The QP uses the total derivative of the implemented desired-wrench command $w_d$. The following expressions illustrate its evaluation for a translational PID law and a geometric attitude controller~\cite{lee2010}.

\subsection{Translational PID Wrench in Body Coordinates}
Define $e_p=p_d-p$, $e_v=v_d-v$, and $\dot\eta_p=e_p$. With constant gains and inertial compensation force $F_c^I$, let
\begin{equation}
F_d^I=m(a_d+K_pe_p+K_ve_v+K_I\eta_p+ge_3)+F_c^I.
\end{equation}
Then
\begin{equation}
\dot F_d^I=m\{j_d+K_pe_v+K_v(a_d-\dot v)+K_Ie_p\}
+\dot F_c^I.\label{eq:forceIderivative}
\end{equation}
Since the allocator uses body-frame wrench,
\begin{equation}
F_d=R^\top F_d^I,\qquad
\dot F_d=-\sk\Omega F_d+R^\top\dot F_d^I.
\label{eq:forceBderivative}
\end{equation}
The rotating-frame term is required during body reorientation.

The compensation derivative $\dot F_c^I$ is included when an NRIC or another additive compensation block is used.

\subsection{Geometric Attitude Controller}
Define
\begin{align}
Q&=R^\top R_d,\qquad
e_R=\tfrac12(R_d^\top R-R^\top R_d)^\vee,\\
e_\Omega&=\Omega-Q\Omega_d,
\qquad E(Q)=\tfrac12\{\operatorname{tr}(Q)I-Q\}.
\end{align}
The error derivatives satisfy
\begin{equation}
\dot e_R=E(Q)e_\Omega,\qquad
\dot Q=-\sk\Omega Q+Q\sk{\Omega_d}.
\end{equation}
To express the torque derivative compactly, let
\begin{align}
a_\Omega&=\dot\Omega,\quad a_{\Omega,d}=\dot\Omega_d,
\quad j_{\Omega,d}=\dot a_{\Omega,d},\nonumber\\
v_R&=Q\Omega_d,\qquad q_R=Qa_{\Omega,d},\qquad
b_R=\Omega\times v_R-q_R.
\end{align}
Then
\begin{align}
\dot v_R&=-\Omega\times v_R+q_R,\label{eq:vdotR}\\
\dot q_R&=-\Omega\times q_R+
Q(\Omega_d\times a_{\Omega,d}+j_{\Omega,d}),\\
\dot e_\Omega&=a_\Omega-\dot v_R,\\
\dot b_R&=a_\Omega\times v_R+\Omega\times\dot v_R-\dot q_R.
\end{align}
For a geometric torque controller with additive compensation $\tau_c$,
\begin{equation}
\tau_d=-K_Re_R-K_\Omega e_\Omega+
\Omega\times J_b\Omega-J_bb_R+\tau_c,
\end{equation}
the full derivative is
\begin{align}
\dot\tau_d={}&-K_RE(Q)e_\Omega-K_\Omega(a_\Omega-\dot v_R)\nonumber\\
&+a_\Omega\times J_b\Omega+\Omega\times J_ba_\Omega
-J_b\dot b_R+\dot\tau_c.\label{eq:taudot}
\end{align}
Thus $\dot w_d=[\dot F_d^\top,\dot\tau_d^\top]^\top$. If an attitude integral term is used, its derivative must be added consistently. Body angular acceleration can be obtained from the dynamics or an estimator, but its uncertainty cannot be ignored in an exact-derivative claim.

\bibliographystyle{IEEEtran}
\bibliography{references}
\end{document}